%% file: cvpr.tex
\documentclass[final]{cvpr}

\usepackage{times}
\usepackage{epsfig}
\usepackage{graphicx}
\usepackage{subfigure}
\usepackage{amsmath}
\usepackage{amssymb}
\usepackage{amsthm}

\usepackage{algorithm}  
\usepackage{algorithmicx}  
\usepackage{algpseudocode}
\usepackage{enumitem}
\usepackage{mathtools}
\usepackage{xcolor}
\usepackage{color, colortbl}
\definecolor{Gray}{gray}{0.93}

\usepackage[numbers]{natbib}

\usepackage[pagebackref=true,breaklinks=true,colorlinks,bookmarks=false]{hyperref}

\def\cvprPaperID{20} 
\def\confYear{CVPR 2021}

\input{math_commands.tex}

\begin{document}
\pagenumbering{gobble}
\title{Towards Fair Federated Learning with Zero-Shot Data Augmentation}

\author{Weituo Hao$^1$, Mostafa El-Khamy$^2$, Jungwon Lee$^2$, Jianyi Zhang$^1$, \\Kevin J Liang$^1$, Changyou Chen$^3$, Lawrence Carin$^1$ \\
$^1$Duke University\quad $^2$Samsung\quad $^3$State University of New York at Buffalo\\
{\tt\small weituo.hao@duke.edu}
}

\maketitle

\begin{abstract}
Federated learning has emerged as an important distributed learning paradigm, where a server aggregates a global model from many client-trained models, while having no access to the client data. Although it is recognized that statistical heterogeneity of the client local data yields slower  global model convergence, it is less commonly recognized that it also yields a biased federated global model with a  high variance of accuracy across clients.
In this work, we aim to provide federated learning schemes with improved fairness. 
To tackle this challenge, we propose a novel federated learning system that employs zero-shot data augmentation on under-represented data to mitigate statistical heterogeneity, and encourage more uniform accuracy performance across clients in federated networks. We study two variants of this scheme, Fed-ZDAC (federated learning with zero-shot data augmentation at the clients) and Fed-ZDAS (federated learning with zero-shot data augmentation at the server). Empirical results on a suite of datasets demonstrate the effectiveness of our methods on simultaneously improving the test accuracy and fairness.

\end{abstract}


\section{Introduction}
Major advances in deep learning over the last decade have in large part been possible due to the increasing availability of data.
With the proliferation of personal computers, smart phones, and edge devices, data are being generated and collected at unprecedented rates, providing the large datasets needed to train the machine learning that power ``intelligent'' services that are becoming increasingly common in daily life. 
However, the rich content in these data that enables such smart behavior may also be revealing of personal information.
Traditional learning methods pool the data into a central repository for training, which makes personal data vulnerable to breaches or interception.

Federated learning~\cite{mcmahan2016communication} has emerged as an alternative strategy, with an emphasis on user data privacy.
In the federated learning paradigm, learning takes place on the client devices themselves, which means that the user's personal data never leaves the local device.
In place of the data, the updated model itself is sent to a coordinating server, which then aggregates the updates and distributes the new model to the clients. 

\begin{figure}[t]
	\centering
	\includegraphics[width=0.45\textwidth]{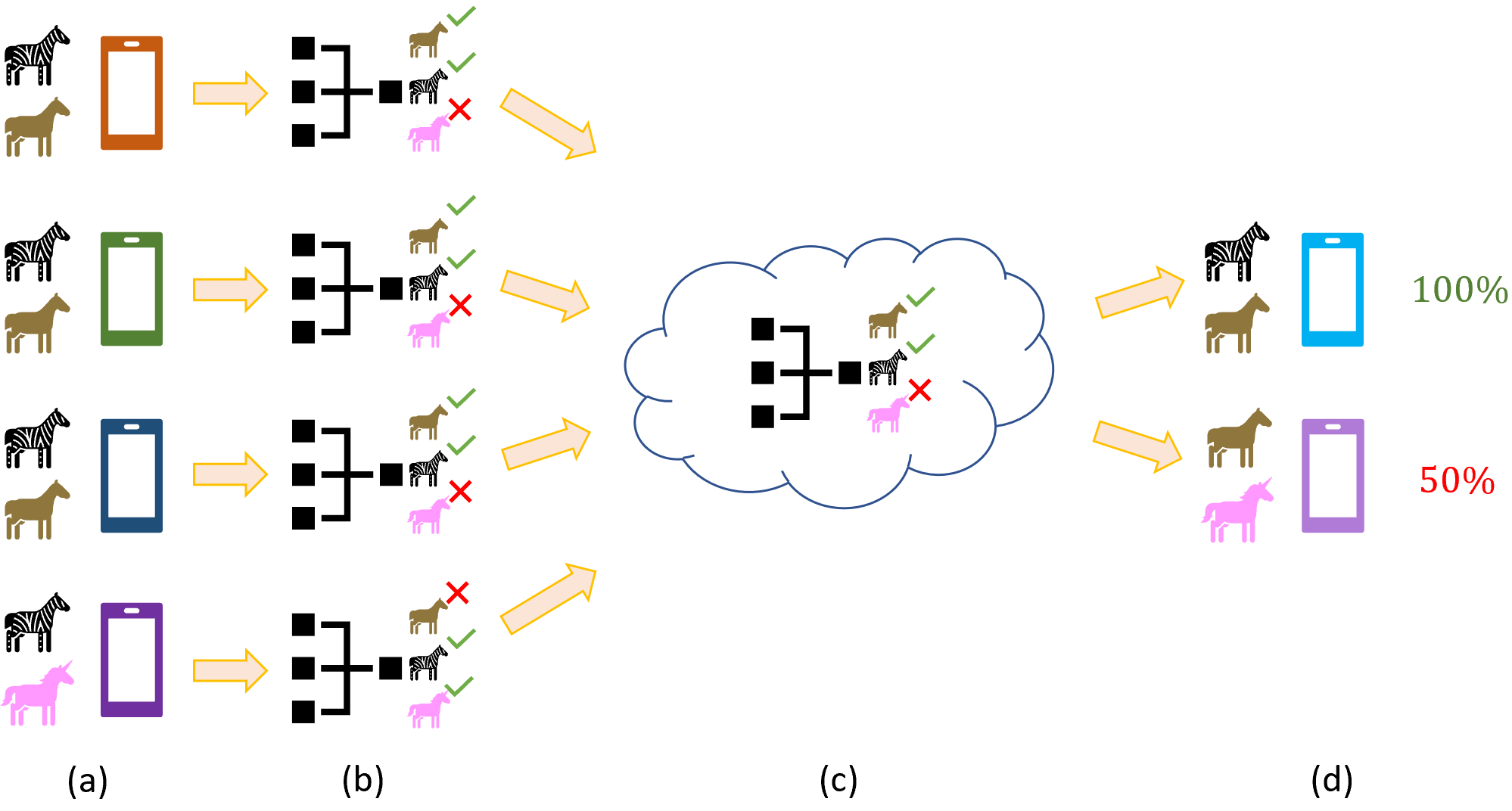}
	\caption{
    (a) Client data statistical heterogeneity.
	(b) Different model characteristics after local update.
	(c) Model aggregation at the server, which may drown out minority clients.
	(d) One-size-fits-all global model performs well in general, but poorly on minority clients.
	}
	\label{fig:FL_illu}
\end{figure}

While federated learning has demonstrated promise for user data privacy, a major challenge is statistical heterogeneity~\cite{li2020federated, li2018fedprox, arivazhagan2019federated, hao2020waffle}: data distributions between clients may exhibit significant differences. 
These differences may lead to variance in learned local models after training on each client's local data.
Additionally, the formulation in \cite{mcmahan2016communication} is fundamentally a one-size-fits-all solution, meaning the learned global model may perform worse for some clients, as shown in Figure~\ref{fig:FL_illu}.
As a result of these factors, federated learning methods tend to perform poorer when the data are not independent and identically distributed (\textit{i.i.d.}) among clients~\cite{mcmahan2016communication, zhao2018federated}.

What is more concerning, however, is that the accuracy loss due to statistical heterogeneity may be borne unequally among clients~\cite{li2019fair}.
In populations with unequally sized subgroups, clients with less common classes tend to see worse performance~\cite{hao2020waffle}.
This may be, in part, due to catastrophic forgetting~\cite{McCloskey1989,shoham2019overcoming}: clients from outside a subpopulation have a tendency to forget features not found in their own data and, during aggregation, the less represented clients may have their learned features drowned out when the model weights are averaged.
In the real world, these client characteristics may represent ethnicity~\cite{klare2012face}, gender~\cite{albiero2020analysis, leavy2018gender}, age~\cite{albiero2020does}, language~\cite{gu2018universal}, dialect, demographics, animal species, or disease trait.
Therefore, the inability to cope with statistical heterogeneity my lead to potentially unfair algorithms, that provide in- accurate classifications based on certain characteristics of their input data.
A popular and effective strategy for preventing forgetting is replay~\cite{lopez2017gradient, Shin2017}: storing a small buffer of samples for rehearsal.
In federated learning, however, clients do not have access to data from parts of the distribution that are not well-represented in their own data.
This is, in part, by design, as client data are kept private and local to the device.

In this work, we propose a federated learning system with zero-shot data augmentation (Fed-ZDA) to generate pseudo-exemplars of unseen classes, without having access to the private data.
Such a strategy preserves the model's ability on previously sampled client data when learning the local client update.
This makes the model less likely to lose representational ability for parts of the distribution that are rarer.
We explore two strategies for using zero-shot data augmentation for federated learning, one in which synthetic samples are generated at the client (Fed-ZDAC), and another where they are generated at the server (Fed-ZDAS). Both methods are illustrated in Figure~\ref{fig:zdac_zdas}.
Differential privacy analysis shows that our proposed approach satisfies $(0,\delta)$ differential privacy.
Finally, experiments on MNIST~\cite{Lecun1998}, FMNIST~\cite{xiao2017fashion}, and CIFAR-10~\cite{Krizhevsky2009} show that both Fed-ZDAC and Fed-ZDAS result in more equitable model performance.

\section{Related Work}
\subsection{Federated Learning} \label{subsec:shfl}
\paragraph{Statistical Heterogeneity}
Statistical heterogeneity of the data distributions of client devices has long been recognized as a challenge for federated learning~\cite{zhao2018federated}. Despite acknowledging statistical heterogeneity, many federated learning algorithms still focus on learning a single global model~\cite{mcmahan2016communication}; such an approach often suffers from divergence of the model, as local models may vary significantly from each other. To address this challenge, a number of works break away from the single-global-model formulation. Several~\cite{smith2017federated,corinzia2019variational} have cast federated learning as a multi-task learning problem, with each client treated as a separate task. FedProx~\cite{li2018fedprox} adds a proximal term to account for statistical heterogeneity by limiting the impact of local updates. In \cite{zhao2018federated} performance degradation from skewed data is recognized, proposing global sharing of a small subset of data which, while effective, may compromise privacy.

\paragraph{Fairness} There has been rising interest in developing fair methods for machine learning~\cite{yucer2020exploring}. However, such concerns have been less addressed in federated learning. A commonly used fairness definition has been proposed in \cite{zafar2017fairness}. However, it forces the accuracy to be identical on each device across hundreds to millions of clients, given the significant variability of data in the network. Recent work~\cite{li2019fair} has taken a step towards
addressing this by introducing uniformity to describe the fairness in federated learning, in which the goal is instead to ensure that the underfit groups are assigned more weight in the global learning objective. However, the proposed objective causes a performance drop in clients who could have better results under traditional federated average objective, which may reduce these clients' incentive to participate the federated learning process. The work in \cite{hao2020waffle} proposed rank-one factorization on model parameters to ensure consistent model performance across clients, by leaving factors locally. However, this Bayesian approach usually costs more training time, and development of client-specific models is beyond the single-global-model focus of this paper.

\subsection{Zero-Shot Data Augmentation} \label{subsec:mi}
Deep learning performance is highly dependent on the quantity of data available~\cite{halevy2009unreasonable, sun2017revisiting}. Data augmentation, which inflates the size of a dataset without necessitating further data collection, has proven effective in a wide range of settings~\cite{Lecun1998,krizhevsky2012imagenet, zhang2018mixup, liang2021mixkd}, improving machine learning model generalization.
However, most data augmentations apply transformations to the existing data, thus making the implicit assumption that at least some data is available.
These techniques are thus difficult to apply when no data is available. Consequently, \textit{DeepInversion}~\cite{yin2020dreaming} proposes a data-free knowledge transfer based on synthesizing data, effectively providing more teacher behavior for a student to learn. Also, \cite{cai2020zeroq} proposes a similar method for network quantization, by updating random input to match stored batch norm layer statistics. In our work, since the server has no access to the local data, synthesizing a reasonable amount of fake data for deficient classes would encourage a more fair global model. Also, unlike the work~\cite{jeong2018communication,zhao2018federated} which violates the rule that clients should never share data to other clients or the server, zero-shot data augmentation synthesize data based on the model information only. 
Note that using synthesized samples for data augmentation differs from related works like \cite{goetz2020federated}, which take an approach similar to dataset distillation~\cite{wang2018dataset} to synthesize data for the purpose of compressing model updates for communication efficiency purposes.

\subsection{Differentially Private Federated Learning} \label{subsec:dp}
With the increasing awareness of data security and confidential user information, privacy has become an important topic for machine learning systems and algorithms. In order to solve this issue, differential privacy has been proposed to prevent revealing training data~\cite{abadi2016deep}. Even though federated learning enables local training without sharing the data to the server, it is still possible for an adversary to infer the private information to some extent, by analyzing the model parameters after local training~\cite{wang2019beyond,ma2020safeguarding}. Therefore, combining differential privacy with federated learning has been studied in many previous works. To ensure federated learning approaches satisfy differential privacy, the work in \cite{geyer2017differentially} proposed a client level perspective by adding Gaussian noise to the model update, which can prevent the leakage of private information  and achieve good privacy performance. In \cite{truex2019hybrid}, a combination of differential privacy and secure multiparty computation was proposed to block differential attacks. However, previous approaches based on adding noise to model parameters struggle to capture the appropriate trade-off between the model performance and privacy budget. Our proposed zero-shot data augmentation can be interpreted as a new randomization mechanism different from adding Gaussian noise, satisfying differential privacy without hurting model performance.

\section{Federated Learning with Zero-Shot Data Augmentation}
\begin{figure*}[t!]
    \centering
    \begin{subfigure}
        \centering
        \includegraphics[height=1.5in]{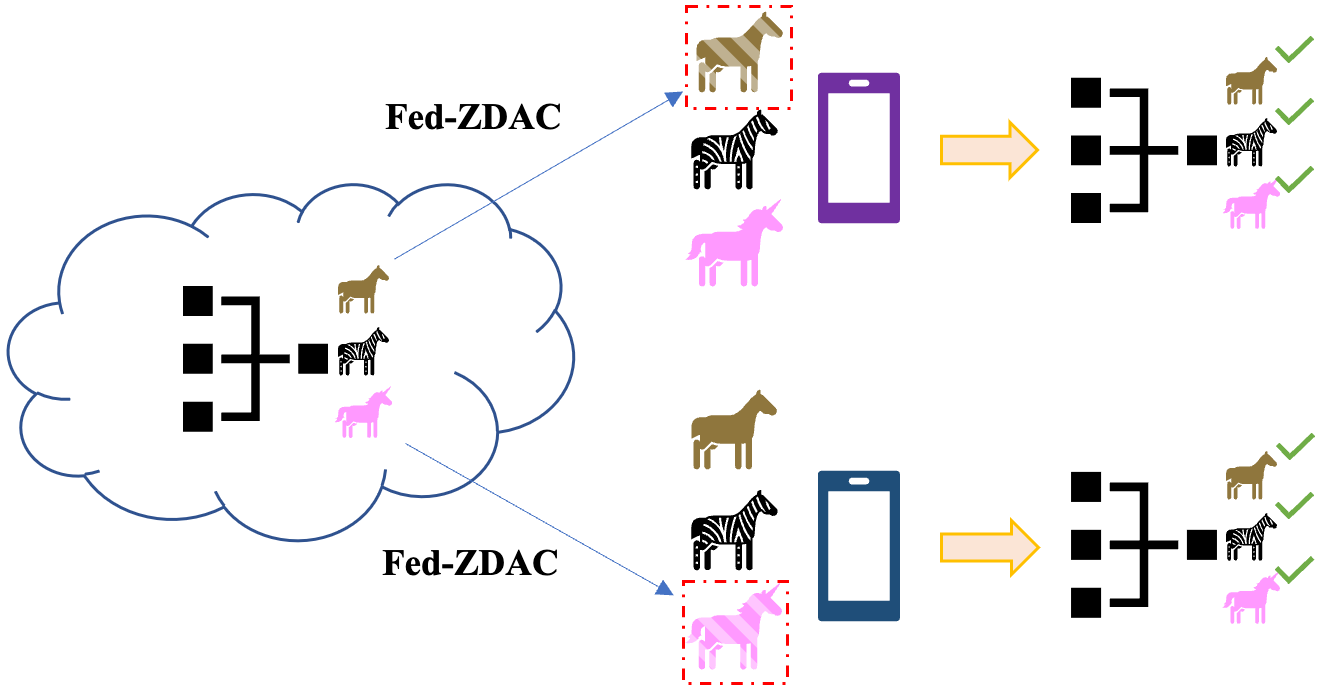}
    \end{subfigure}%
    \hfill
    \begin{subfigure}
        \centering
        \includegraphics[height=1.5in]{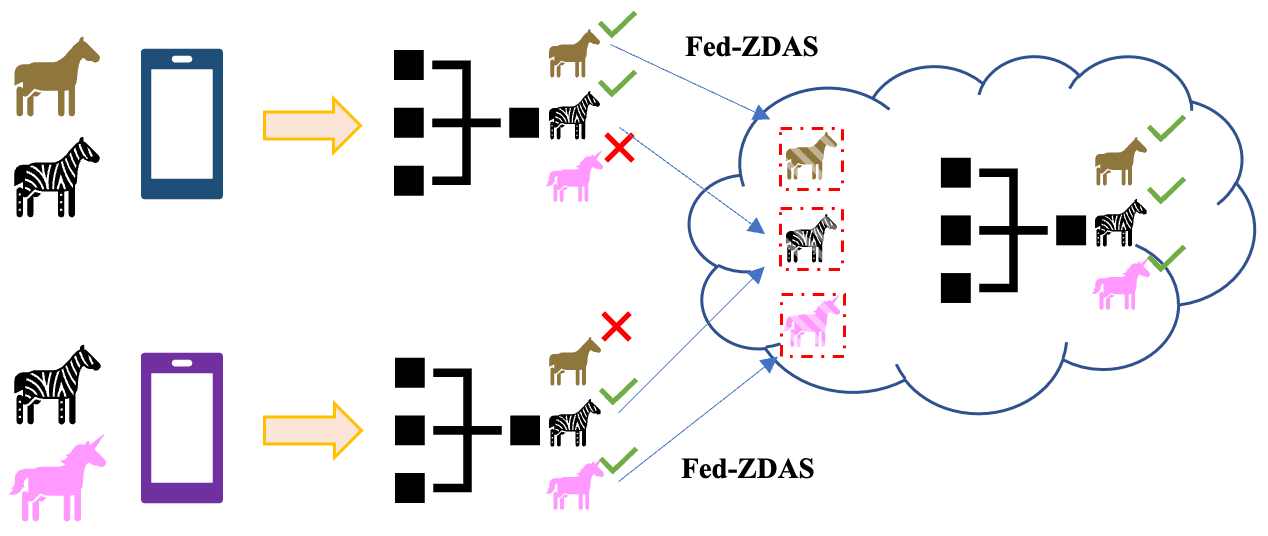}
    \end{subfigure}
    \caption{Illustration of Fed-ZDAC (left) and Fed-ZDAS (right). In Fed-ZDAC, clients train the model after data augmentation. In Fed-ZDAS, the server distributes the model after training on augmented fake data.}
    \label{fig:zdac_zdas}
\end{figure*}

We propose a federated learning method with zero-shot data augmentation (Fed-ZDA), for the purpose of improving the robustness and fairness of federated learning.
To improve the fairness of the global model, Fed-ZDA introduces new synthetic data, generated either at the server or at the client nodes, to supplement training with underrepresented samples.
Notably, these samples are generated without access to user data, but rather from shared models post-local update.
We start by reviewing standard federated learning, which Fed-ZDA builds on. We then describe the zero-shot data-augmentation method we use for Fed-ZDA. We describe two deployments of Fed-ZDA, Fed-ZDAC and Fed-ZDAS, where the zero-shot data-augmentation is done at the client nodes and at the server node, respectively. 

\subsection{Federated Learning}\label{subsec:fl}
In its most basic form, the federated learning objective is commonly expressed as the following:
\begin{equation}\label{eq:fl}
    \min_{w} f (w) =\sum_{i=1}^{Z}p_{i}F_{i}(w)
\end{equation}
where $F_{i}(w) \coloneqq \bbE_{\vx_{i} \sim \calD_i}[f_{i}(w;\vx_{i})]$ is the local objective function of the $i^{\mathrm{th}}$ client, $Z$ is the number of devices or clients, and $p_{i} \geq 0$ is a weight assigned to the $i^{\mathrm{th}}$ client. 

Standard federated learning aims to aggregate, at the centralized server, a federated global model from the client models, typically by averaging them. In this scenario, the  clients only share their trained models with the server, and do not share the datasets on which their models have been trained. The server and the client communicate for $T$ rounds to update the global model $M$. 
A single communication round contains three main steps: 
\begin{enumerate} 
\item The server randomly samples a subset of clients and distributes the model to the sampled clients. 
\item Each sampled client updates the model by training it with their local training data.  
\item Each client sends their updated model back to the server and the server aggregates the received client models into a new global model. 
\end{enumerate}

Typically, learning aggregates the models by federated averaging (FedAvg), in which the federated global model $M_{\mathcal{G},t}$ aggregated at the $t^{\mathrm{th}}$ communication round is simply a weighted average of all the client models received at this round. Let $M_{i,t}$ be the model trained by the $i^{\mathrm{th}}$ client $C_i$ at the $t^{\mathrm{th}}$ communication round, and $\mathcal{S}_t$ is the set of indices of the sampled clients at the $t^{\mathrm{th}}$ round.
\begin{equation} \label{eq:fedavg}
    M_{\mathcal{G},t} = \sum_{i \in \mathcal{S}_t} w_i M_{i,t}
\end{equation}
Different weights $0 \leq w_i \leq 1$ can be assigned to the clients depending on different factors, such as the amount of data they have been trained on, if such information is known at the server. Otherwise, a simple arithmetic mean is adopted.

Ideally, after sufficient communication rounds, the global model should converge to a solution that has learned using the data from all clients. However,  heterogeneous data distributions across clients may cause inconsistent model performance. In particular, if the dataset distributions of the different clients are  skewed towards a majority group of classes, FedAvg may result in a model with a large variance in accuracy across classes, resulting in a large variance in the global model accuracy on the data of different clients. Hence, standard federated learning suffers from the notion of unfairness towards the under-represented clients, providing poor accuracy on their data.

\subsection{Zero-Shot Data Generation} \label{subsec:dg}
Data augmentation has proven effective in many machine learning settings, such as when there is data scarcity or class imbalance. Commonly used techniques include performing transforms (\textit{e.g.} rotations, flips, crops, added noise) based on the original true data, combinations in feature space, and synthesizing data by generative models. However, these techniques require access to training data or at least a few data sample seeds.
In federated learning, these are not available, as data never leaves the individual clients, making conventional augmentation techniques challenging. 
In this work, we propose zero-shot data generation (ZSDG), to generate labeled synthetic data for data-augmentation at the clients, without having any access to any training data. This approach utilizes trained models (either the global model pre-update, or the local models post-update) to generate synthetic data of the desired classes without access to \textit{any} non-local data.

One way to generate synthetic data whose statistics match those of the original training data is to find the data that results in similar statistics as those stored in the batch normalization (BN) layers of the pretrained model. However, without assigning class labels to this data, one cannot use this data in a data-augmentation regime for supervised training. For data augmentation with $N$ possible classes, we generate data for each class $1 \leq n \leq N$, represented by its corresponding one-hot vector $\Bar{y}(n)$, which has $1$ at the $n^{\mathrm{th}}$ index and zero otherwise.
Let model $M$ be a neural network with $L$ layers. For simplicity of notation, assume the model has $L$ batch normalization (BN) layers and denote the activation before the ${\ell}^{\mathrm{th}}$ BN layer to be $z_{\ell}$. The ${\ell}^{\mathrm{th}}$ BN layer is parameterized by a mean $\mu_{\ell}$ and variance $\sigma_{\ell}$ calculated from the input feature maps when the model was being trained.  During the forward propagation, $z_{\ell}$ is normalized with the  parameters of the BN layer. Note that given a pretrained model $M$, batch norm statistics of all BN layers are stored and accessible.
  Given a target class $\Bar{y}(n)$ the ZSDG reduces to the optimization problem that finds the input data $\Bar{x}$ that result in the batch norm statistics matching those stored in the BN layers of the pretrained model, and are classified by the pretrained model as having label $\Bar{y}(n)$. Given the pretrained model $M$, with BN statistics $\mu_{\ell}$ and $\sigma_{\ell}$ stored in its layers $1 \leq \ell \leq L$, the ZSDG optimization problem to generate synthetic labeled data  $(\Bar{x}(n), \Bar{y}(n))$ for ${n\in\{1,2, \cdots, N\}}$ can be expressed as:
\begin{align}\textstyle \nonumber
   \Bar{x}(n)  =  \arg \min_{\Bar{x}} & \sum_{\ell=1}^{L}  \left\Vert \Bar{\mu}_{\ell} - \mu_{\ell} \right\Vert_{2}^{2} + \left\Vert \Bar{\sigma}_{\ell} - \sigma_{\ell} \right\Vert_{2}^{2}   \\& + \mathcal{H}(M(\Bar{x}), \Bar{y}(n)), \label{eq:cond_recon}   
\end{align} where  $\Bar{\mu}_{\ell}$, and $\Bar{\sigma}_{\ell}$ are, respectively, the mean and standard deviation evaluated at layer $\ell$ with the generated input data,  $M(\Bar{x})$ denotes the model classification output when the  input is $\Bar{x}$, and $\mathcal{H}$ is the cross entropy loss function to learn the class labels. To solve Equation~\ref{eq:cond_recon} for a selected class $\Bar{y}(n)$,  an input is initialized randomly from a normal distribution and, then, updated using gradient descent, while fixing the model parameters during back-propagation. The ZSDG is described in Algorithm~\ref{alg:zsdg}.  

\begin{algorithm}[h]
    \begin{algorithmic}[1]
	\caption{Zero-Shot Data Generation (ZSDG)}
	\label{alg:zsdg}
	\State {\bf Input:} Model $M$ with $L$ batch normalization layers
	\State {\bf Output:} A batch of labeled fake data: $(\Bar{x}, \Bar{y})$  
	\State Get $\mu_{\ell}$, $\sigma_{\ell}$ from Batch Normalization layers of $M$, $\ell \in \{1,2,\cdots, L\}$

\For  {$n=1,2, \cdots, N$}
	    \State Generate $\Bar{x}(n)$ randomly  from a Gaussian distribution, assign it a label  $\Bar{y}(n)$
\EndFor
\For  {$j=1,2, \cdots$}
	    \State Forward propagate $M(\Bar{x}(n))$ for all $n$
	    \State Gather intermediate activations $\Bar{z}_{\ell}$, $\ell \in \{1,2,..., L\}$
	    \State Gather BN statistics: $\Bar{\mu}_{\ell}$ and $\Bar{\sigma}_{\ell}$ induced by intermediate activations $\Bar{z}_{\ell}$, $\ell \in \{1,2,..., L\}$ 
        \State Compute the loss based on Equation~\ref{eq:cond_recon}
        \State Backward propagate and update the input $\Bar{x}(n)$ only
\EndFor
\State {\bf Return} $(\Bar{x}, \Bar{y}) = \cup_{n\in\{1,2, \cdots, N\}}(\Bar{x}(n), \Bar{y}(n))$ 	
	
	\end{algorithmic}
\end{algorithm}

\subsection{Zero-Shot Data Augmentation at Clients} \label{subsec:zsda_client}
It is common to have statistical heterogeneity in the training data across clients. To address the deficiency of their training data in some classes, and promote the global model fairness, clients are instructed to augment their training data with fake data using ZSDG, before updating the received global model. 
Let the $i^{\mathrm{th}}$ client at the $t^{\mathrm{th}}$ communication round have the real local training data  $(x_i, y_i)_t$ with input and label pairs. Let $(\Bar{x}_i, \Bar{y}_i)_t$ be the synthetic (fake) data generated using ZSDG over all classes from the received global model $M_{\mathcal{G},t-1}$. Then the procedure for federated learning with zero-shot data augmentation at the clients (Fed-ZDAC) is described by Algorithm~\ref{alg:zsda_c}.

\begin{algorithm}[t]
    \begin{algorithmic}[1]
	\caption{Fed-ZDAC: Federated Learning with Zero-Shot Data Augmentation at Clients}
	\label{alg:zsda_c}
	\State {\bf Input:}  Communication rounds $T$, global model $M$
	\For  {$t=1, \cdots, T$}
	    \State Server randomly selects subset $\mathcal{S}_t$ of clients
	    
	    \State Server sends $M_{\mathcal{G},t-1}$ to $\mathcal{S}_t$
	    \For  {Clients $C_i$, $i \in \mathcal{S}_t$ \textbf{in parallel}}
            \State Generate labeled fake data $(\bar{x}_i, \bar{y}_i)_t$  by ZSDG from the global model $M_{\mathcal{G},t-1}$ 
    	    \State Client $C_i$ produces the model $M_{i,t}$ by updating the model $M_{\mathcal{G},t-1}$ with the mix of real local data available at round $t$, and the fake ZSDG data: $\{(x_i, y_i)_t, (\bar{x}_i, \bar{y}_i)_t \}$ 
    	    \State Send the updated client model $M_{i,t}$ to the server.
        \EndFor
        \State Server aggregates all client models $M_{i,t}$, $i \in \mathcal{S}_t$, e.g. by Equation~\ref{eq:fedavg}, to obtain the updated $M_{\mathcal{G},t}$
	\EndFor
	\end{algorithmic}
\end{algorithm}

\subsection{Zero-Shot Data Augmentation at Server} \label{subsec:zsda_server}
In Section~\ref{subsec:zsda_client}, we discussed federated learning with data augmentation at the client nodes. In practice, clients may be mobile computing devices that are limited in their computing resources and storage capacity, which may restrict their capacity for data augmentation. Clients may also not care about fairness of the global model towards other clients, and would like to train the best model for their classes of interest only.  
It is also in the best interest of the server to produce a fair and accurate model, that does not ignore data classes of the under-represented clients.  
In addition, if the global model is fair, and  each client updates the global model from the same fair initialization, federated learning can convergence faster to a fair solution. Consequently, we propose federated learning with zero-shot data augmentation at the server (Fed-ZDAS). 
We use the same notation as described in Section~\ref{subsec:zsda_client}. In more detail, the server distributes its global model to a subset of clients, Each of these clients update this global model with their local training data $(x, y)$ and send it back  to the server. 
In strive for fairness, the server will generate equal amount of fake data from each received client model, and combines all fake client data into a balanced synthetic dataset. 
The server aggregates all received client models into a single model, and then trains the single model by the combined synthetic dataset. To our knowledge, this is the first federated learning protocol which involves training at the server, since in general the server is assumed not to have any data.
Fed-ZDAS is described in Algorithm~\ref{alg:zsda_s}.

\begin{algorithm}[t]
    \begin{algorithmic}[1]
	\caption{Fed-ZDAS: Federated Learning with Zero-Shot Data Augmentation at the Server}
	\label{alg:zsda_s}
	\State {\bf Input:}  Communication rounds $T$, global model $M$
	\For  {$t=1, \cdots, T$}
	    \State Server randomly selects subset $\mathcal{S}_t$ of clients
	    
	    \State Server sends $M_{\mathcal{G},t-1}$ to $\mathcal{S}_t$
	    \For  {Clients $C_i$, $i \in \mathcal{S}_t$ \textbf{in parallel}}
	      \State Client $C_i$ produces the model $M_{i,t}$ by updating the model $M_{\mathcal{G},t-1}$ with its real local data available at round $t$ $(x_i, y_i)_t$ 
    	    and sends the updated client model $M_{i,t}$ to the server.
    	 \EndFor
        
          	\State Server generates a class-balanced fake labeled data $(\bar{x}_i, \bar{y}_i)_t$ by ZSDG with each received client model $M_{i,t}$, $i \in \mathcal{S}_t$. 
    	\State Server combines the fake  data generated with the different client models into a combined balanced dataset $(\bar{x}, \bar{y})_t = \cup_{i \in \mathcal{S}_t} (\bar{x}_i, \bar{y}_i)_t.$
    	\State Server aggregates all client models $M_{i,t}$, $i \in \mathcal{S}_t$, e.g. by Equation~\ref{eq:fedavg}, to obtain an interim global model $\tilde{M}_{\mathcal{G},t}$
        \State Server trains $\tilde{M}_{\mathcal{G},t}$ using the combined fake dataset $(\bar{x}, \bar{y})_t$ to produce the updated global model $M_{\mathcal{G},t}$ 
        \EndFor
	\end{algorithmic}
\end{algorithm}

Since the motivation of federated learning is protecting client data privacy, we also prove that our proposed method satisfies client-level differential privacy (DP), a local differential privacy adopted as \cite{wei2020federated}. Intuitively, before clients send updated model parameters back to the server, we seek for a randomized perturbation on these model parameters such that the server can not distinguish if certain client has been involved in the current communication round. A standard way to satisfy differential privacy is adding Gaussian noise to model parameters with trade-off between model's performance and privacy budget.~\cite{abadi2016deep,wei2020federated}. In contrast, our method can be considered as a kind of perturbation to model parameters with useful information as opposed to pure random noise. As a result,   
we show that our proposed method satisfies $(0,\delta)$ differential privacy. For more details about the proof, please check Appendix~\ref{app:dp}.

\section{Experiments} \label{sec:experiments}
\subsection{Datasets and Settings}\label{subsec:data_setting}
\paragraph{Task and Datasets} We conduct experiments on three standard datasets: MNIST~\cite{Lecun1998}, FMNIST~\cite{xiao2017fashion}, and CIFAR-10~\cite{Krizhevsky2009}. Following \cite{mcmahan2016communication} for the federated learning setting, the server selects a proportion $\gamma=0.1$ of 100 clients during each communication round, with $T=100$ total rounds for all methods. Each selected client trains their own model for $E=5$ local epochs with mini-batch size $B=10$. For the data partition, we focus on the non-\textit{i.i.d.} setting, which is typically more challenging and realistic for federated learning. We divide the 60k images into a training set of 50k images and external test set of 10k images, then the training set is distributed to the clients, such that each client only has a subset $Z$ of the classes, and divide their local data set as local training set and local testing set.

Following \cite{hao2020waffle}, we study two data splits, each representing different types of statistical heterogeneity.
The first is unimodal non-\textit{i.i.d.} which is identical to the data partition introduced by \cite{mcmahan2016communication}. The second is multimodal non-\textit{i.i.d.}, in which there exists subpopulations, with some being more prevalent than others. Each subpopulation group can be thought of as a mode of the overall distribution. In other words, the classes are imbalanced in the data set aggregating from all clients' data.

\begin{table*}[h]
  \centering
    \begin{tabular}{|l|l|cc|cc|}
    \hline
     & & \multicolumn{2}{c|}{Unimodal} & \multicolumn{2}{c|}{Multimodal} \\
     Dataset & Method  & Mean Accuracy $\uparrow$ & Variance $\downarrow$ &  Mean Accuracy$\uparrow$ & Variance$\downarrow$\\
    \hline\hline
    \multirow{5}{*}{MNIST}   & FedAvg &97.98$\pm$0.01 &6.70$\pm$1.21 &96.67$\pm$0.73 & 47$\pm$27\\
              & FedProx   &97.93$\pm$0.01&6.33$\pm$1.25 &91.98$\pm$0.80 & 72$\pm$6 \\
 
               & q-FFL  & 95.84 $\pm$0.45&17.00$\pm$9.20 &94.81$\pm$7.55 & 78$\pm$20\\
              \rowcolor{Gray}
              \cellcolor{white}& Fed-ZDAC  & \textbf{98.23}$\pm$0.22 &\textbf{3.54}$\pm$0.85 &\textbf{97.07}$\pm$0.56  & \textbf{27}$\pm$12\\
              \rowcolor{Gray}
              \cellcolor{white}& Fed-ZDAS  & 97.34$\pm$0.61 &6.22$\pm$0.33 &95.49$\pm$0.99  & 49$\pm$22\\
    \midrule
    \multirow{5}{*}{FMNIST} & FedAvg &85.30$\pm$2.67&368$\pm$222 &83.43$\pm$2.28 & 245$\pm$41\\
              & FedProx  &85.64$\pm$2.19& 360$\pm$215 &83.37$\pm$2.04 & 237$\pm$38\\
              
              & q-FFL  &83.09$\pm$0.36&283$\pm$45 &85.97$\pm$0.18 & 175$\pm$10\\
              \rowcolor{Gray}
              \cellcolor{white}& Fed-ZDAC  & 84.65$\pm$2.81 &280$\pm$112 &\textbf{86.00}$\pm$0.07  & 161$\pm$40\\
              \rowcolor{Gray}
              \cellcolor{white}& Fed-ZDAS  & \textbf{86.23}$\pm$2.09 &\textbf{188}$\pm$67 &85.66$\pm$0.85  & \textbf{135}$\pm$11\\
    \midrule
    \multirow{5}{*}{CIFAR-10}   & FedAvg  &\textbf{50.30}$\pm$0.91& 417$\pm$190 &45.53$\pm$1.30 & 288$\pm$98\\
              & FedProx & 49.92$\pm$0.55&416$\pm$186 &\textbf{45.88}$\pm$1.44 & 266$\pm$100\\
              
              & q-FFL & 41.72$\pm$3.00 &\textbf{285}$\pm$115 &38.25$\pm$1.12 & \textbf{243}$\pm$49\\
              \rowcolor{Gray}
              \cellcolor{white}& Fed-ZDAC   &47.18$\pm$1.55&337$\pm$155 &43.92$\pm$1.66 & 244$\pm$70\\
              \rowcolor{Gray}
              \cellcolor{white}& Fed-ZDAS  & 47.78$\pm$1.02 &325$\pm$145 &42.18$\pm$0.81 & \textbf{243}$\pm$64\\
    \bottomrule
    \end{tabular}%
    \caption{Local test performance and client level fairness.}
  \label{tab:run_stats}%
\end{table*}%

\begin{table*}[h]
  \centering
    \begin{tabular}{|l|l|cc|cc|}
    \hline
     & & \multicolumn{2}{c|}{Unimodal} & \multicolumn{2}{c|}{Multimodal} \\
     Dataset & Method  & External Accuracy $\uparrow$ & Variance $\downarrow$ &  External Accuracy$\uparrow$ & Variance$\downarrow$\\
    \hline\hline
    \multirow{5}{*}{MNIST}   & FedAvg &98.02$\pm$0.14 &3.69$\pm$0.55 &93.54$\pm$2.38 & 78$\pm$48\\
              & FedProx   &98.05$\pm0.15$ & 3.69$\pm$0.60 &93.62$\pm$2.38 & 75$\pm$58 \\
 
              & q-FFL  & 95.76 $\pm$0.56 &7.40$\pm$2.02 &92.56$\pm$0.29 & 63$\pm$2\\
              \rowcolor{Gray}
              \cellcolor{white}& Fed-ZDAC  & \textbf{98.21}$\pm$0.08 &\textbf{1.71}$\pm$0.21 &\textbf{95.66}$\pm$0.72  & \textbf{22}$\pm$7\\
              
              \rowcolor{Gray}
              \cellcolor{white}& Fed-ZDAS  & 97.66$\pm$0.08 &2.11$\pm$0.39 &94.10$\pm$0.75  & 40$\pm$16\\
    \midrule
    \multirow{5}{*}{FMNIST} & FedAvg &\textbf{85.03}$\pm$1.54 &435$\pm$296 &79.18$\pm$2.0 & 779$\pm$46\\
              & FedProx  &84.94$\pm$1.19 & 426$\pm$274 &79.13$\pm$1.80 & 794$\pm$33\\
              
              & q-FFL  &80.99$\pm$1.23 & 558$\pm$192 &81.24$\pm$0.43 & 673$\pm$12\\
              \rowcolor{Gray}
              \cellcolor{white}& Fed-ZDAC  & 83.13$\pm$2.56 &263$\pm$101 & \textbf{83.41}$\pm$0.26 & 483$\pm$84\\
              \rowcolor{Gray}
              \cellcolor{white}& Fed-ZDAS  & 83.90$\pm$1.56 &\textbf{260}$\pm$76 & 83.27$\pm$0.25  & \textbf{313}$\pm$68\\
    \midrule
    \multirow{5}{*}{CIFAR-10}   & FedAvg  & 48.89$\pm$1.04& 473$\pm$195 &\textbf{41.74}$\pm$4.30 & 361$\pm$154\\
              & FedProx & 48.83$\pm$0.89 &258$\pm$13 & 37.06$\pm$0.62 & 480$\pm$50 \\
              
              & q-FFL &34.01$\pm$4.46 & 370 $\pm$135 &32.83 $\pm$0.89 & \textbf{218} $\pm$38\\
              \rowcolor{Gray}
              \cellcolor{white}& Fed-ZDAC  &\textbf{49.50}$\pm$0.27 & 378 $\pm$108 & 40.18$\pm$2.59 & 288$\pm$19\\
              \rowcolor{Gray}
             \cellcolor{white} & Fed-ZDAS  & 48.26$\pm$1.02 &\textbf{200}$\pm$69 &39.07$\pm$1.85 & 295$\pm$98\\
    \bottomrule
    \end{tabular}%
    \caption{Global Test Performance and class level fairness.}
  \label{tab:run_stats1}%
\end{table*}%

\paragraph{Model Architecture} Our zero-shot data augmentation requires the model to contain batch normalization layers. For both MNIST and FMNIST, we use a convolutional network consisting of two $5\times5$ convolution layers with 16 and 32 output channels, respectively. 
Each convolution layer is followed by a batch normalization layer and a $2\times 2$ max-pooling operation with ReLU activations. 
A fully connected layer with a softmax is added for the output. For CIFAR-10, we use a convolutional network consisting of two $3\times3$ convolution layers with 16 filters each. Each convolutional layer is followed by a batch normalization layer and a $2\times 2$ max-pooling operation with ReLU activations. These two convolutions are followed by two fully-connected layers with hidden size 80 and 60, with a softmax applied for the final output probabilities. We utilize SGD as the optimizer and set the learning rate as 0.02 for all methods. We compare our methods with three baselines: FedAvg~\cite{mcmahan2016communication}, FedProx~\cite{li2018fedprox} and q-FFL~\cite{li2019fair}.

\subsection{Local Test and Client-Level Fairness} \label{subsec:local_test}
Local test performance is a metric to evaluate the aggregated model on each client's local test set, that is usually class imbalanced. It is an important metric to demonstrate the personalization ability of the aggregated model. As with \cite{li2019fair}, the variance of local test performance across all clients is taken as the fairness metric. Lower variance means the learned model does not lean towards subpopulations who share prevalent data distributions, which is a more fair solution. This metric can be considered as fairness on clients level. We test all methods under both unimodal non-\textit{i.i.d.} and multimodal non-\textit{i.i.d.} The results are listed in Table~\ref{tab:run_stats}. The mean accuracy is the average local test accuracy over all clients and the variance is the client level fairness metric. The standard deviation values are calculated based on the results of different trials by changing random seeds. For MNIST and FMNIST, the proposed method not only achieves the best mean accuracy, but also improves the fairness over all baselines. For CIFAR-10, our method achieves better accuracy than q-FFL and more fairness than FedAvg and FedProx.     
\begin{figure*}[t]
  \subfigure[\label{fig:synmnist}MNIST]{
      \includegraphics[width=0.32\textwidth]{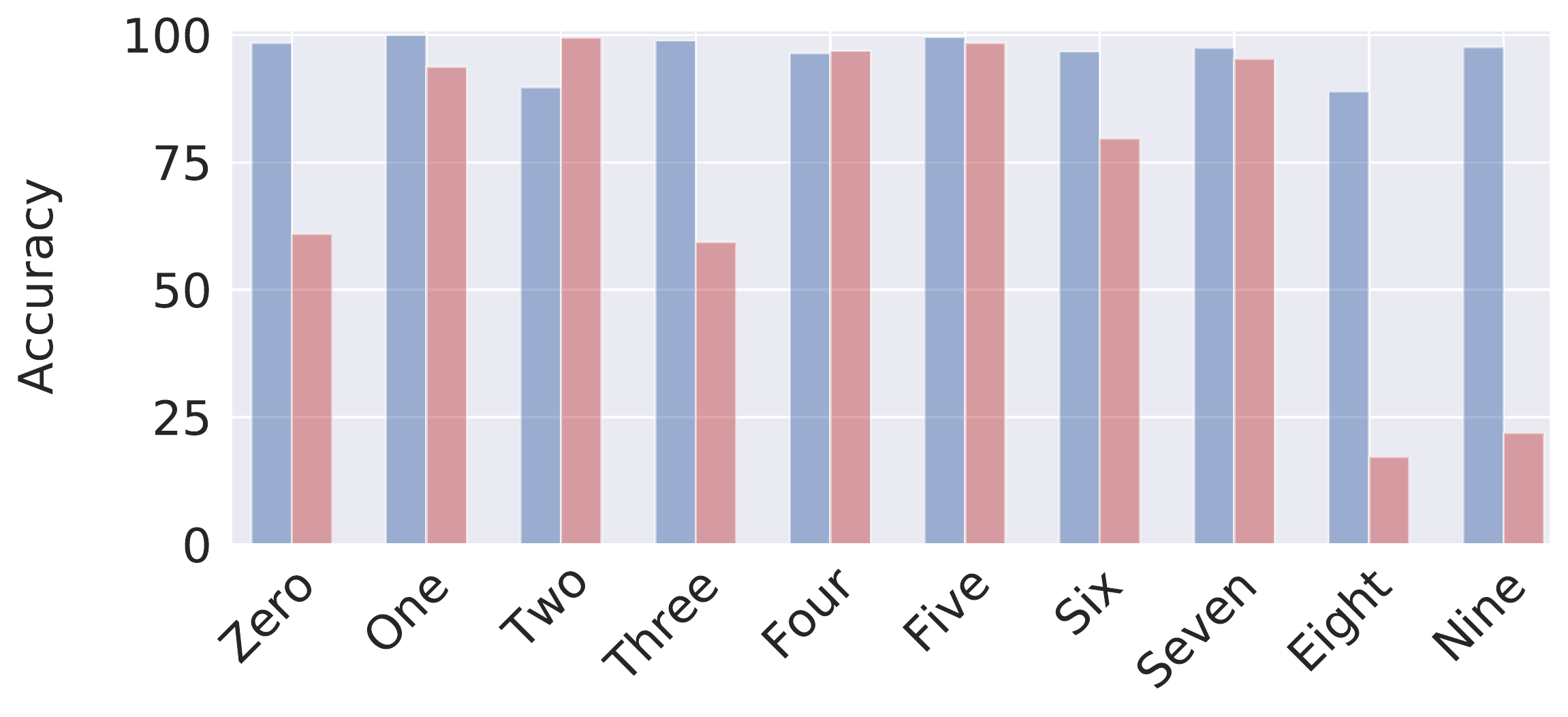}}
  \subfigure[\label{fig:synfmnist}FMNIST ]{
      \includegraphics[width=0.32\textwidth]{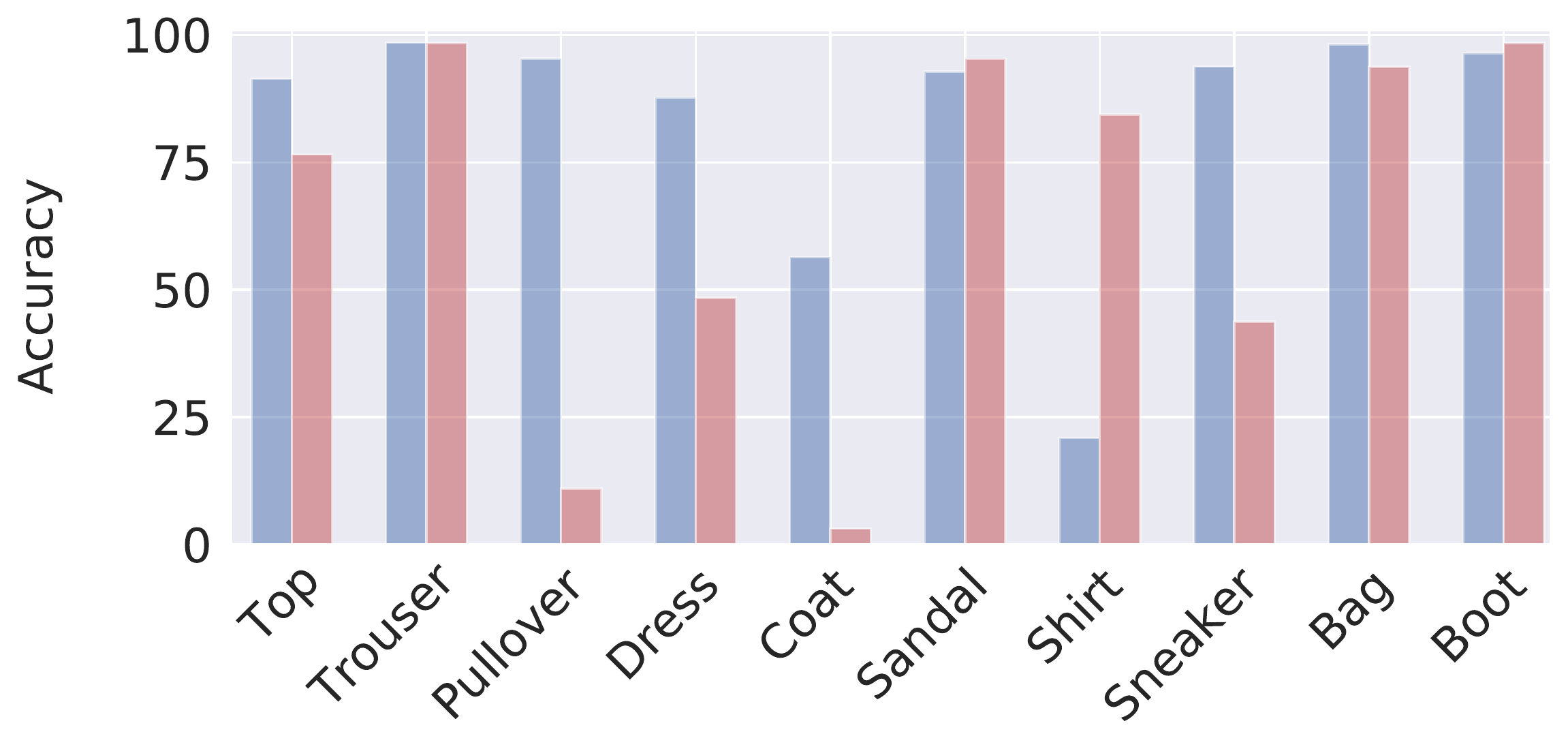}}
  \subfigure[\label{fig:syncifar}CIFAR-10]{
      \includegraphics[width=0.32\textwidth]{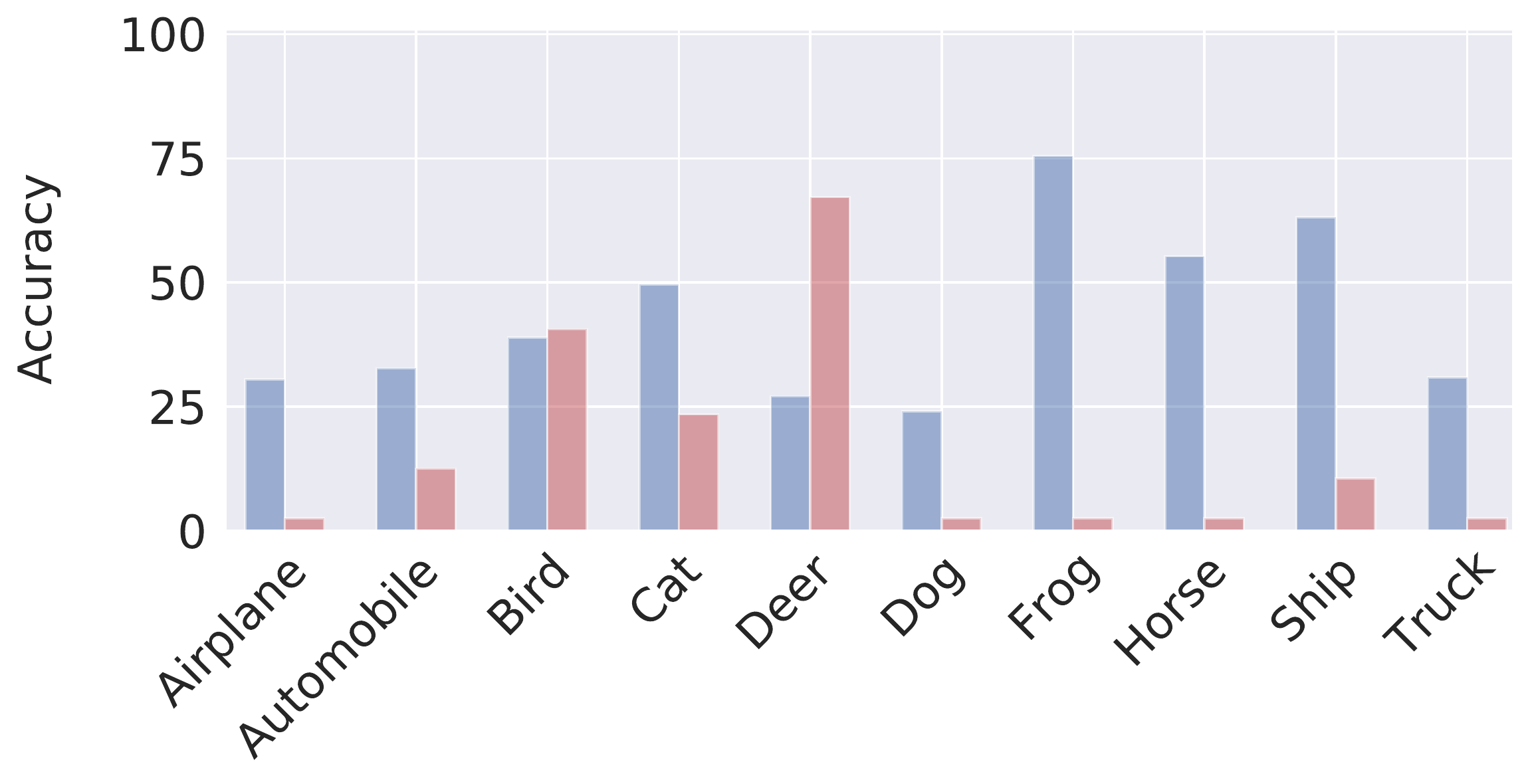}}\\
\caption{\label{fig:quality}The evaluation on augmented data of each class. The blue bars are the trained model's ability. The red bars are the accuracy of the augmented data that is queried from oracle classifiers.}
\end{figure*}

\subsection{Global Test and Class-Level Fairness}\label{subsec:global_test}
Global test performance is a metric to evaluate the aggregated model on an external test set, that is usually class balanced. This is an important criterion to justify the efficiency of the federated learning mechanism and the model's performance on newly coming clients. However, it is still a metric based on average which cannot fully capture whether the model is biased towards, if exists, any prevalent class distribution. We report the variance of accuracy across classes as an extra fairness metric on class level. In Table~\ref{tab:run_stats1}, the external accuracy is the accuracy of the federated model on the held out test set, and the variance is class level fairness metric. We observe better performance on MNIST and FMNIST and comparable results on CIFAR-10. Similarly, all the standard deviation values are calculated based on the results of different trials by changing random seeds. 

\subsection{The Analysis of Augmented Data} \label{subsec:da_analysis} 
The augmented data are generated conditioned on the given label. To study the quality of the synthesized data, we separately trained three classifiers of the same architecture using the optimizer and learning rate described in Section~\ref{subsec:data_setting}, but in a centralized way for MNIST, FMNIST, and CIFAR-10. After training, each classifier achieves test accuracy $99.06\%$, $89.79\%$ and $67.32\%$, respectively. These classifiers are taken as the standalone oracle to evaluate the augmented data. To obtain the synthesized data for test, we run ZSDG based on the model trained in federated learning under multimodal non-\textit{i.i.d.} setting. For each class, we generate 64 images as the test data. The test results for augmented images are shown in Figure~\ref{fig:quality}. We also list the trained model's ability to recognize each class as comparison. In general, the accuracy of synthesized data reflects the ability to `fool' the oracle classifier, \textit{i.e.} the ability to reduce the local data distribution divergence among clients. Since each client owns at most two classes in our experimental setting, the statistical heterogeneity can be mitigated, as long as deficient class of images is synthesized by ZSDG.

\subsection{The Influence of Client Data Distribution} \label{subsec:da_noniid} 
As mentioned in Section~\ref{subsec:da_analysis}, the quality of synthetic data depends on the performance of the model we invert, and the model performance is highly affected by the client data distribution. To further study the influence of the client data distribution on data augmentation quality, we compare the models $f_a$, $f_b$, and $f_c$ learned by three different algorithms:
\begin{itemize}
    \item $f_a$: the model trained by a regular machine learning process on aggregated dataset
    \item $f_b$ : the model trained by federated learning framework on distributed dataset following an \textit{i.i.d} setting
    \item $f_c$: the model trained by federated learning framework on distributed dataset following non-\textit{i.i.d.} setting.Each client has at most 3 out of 10 classes of the images
\end{itemize}
We utilize the standard ResNet34 architecture and train it on CIFAR-10 dataset. The models' performance on test dataset for $f_a$,$f_b$, and $f_c$ are $95.20\%$,$73.96\%$ and $58.10\%$, respectively. In other words, the model's performance decreases as more constraints put in the learning process, which is expected. Consequently, we invert the models and observe the quality of the synthetic images of the same target labels decreases as shown in Figure~\ref{fig:dadc}. This result not only validates that the quality of the synthetic data depends on the base model's performance but also suggests a burn-in stage before model inversion in the federated learning framework which is studied in Section~\ref{subsec:when_da}.

\begin{figure}[t]
  \subfigure[\label{fig:fa}Synthetic images from $f_a$]{
      \includegraphics[width=0.47\textwidth]{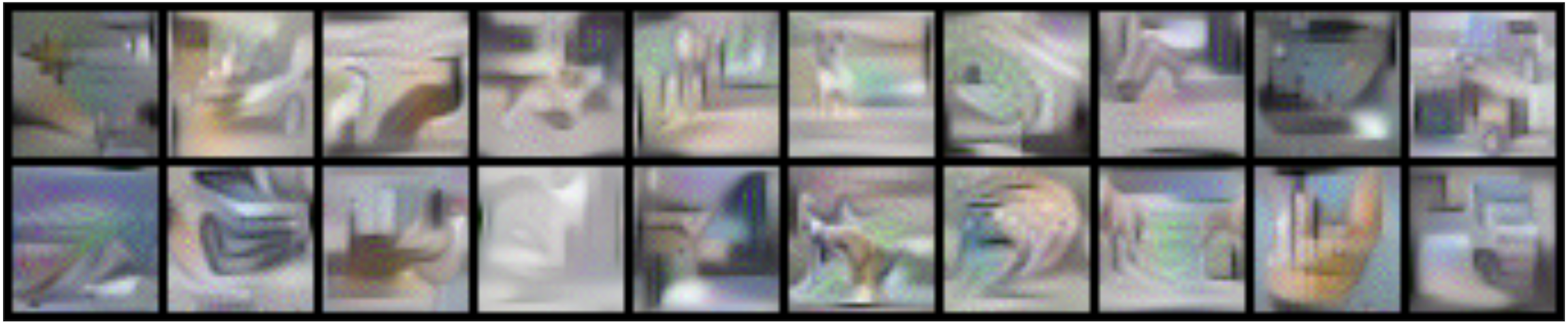}} \\
  \subfigure[\label{fig:fb}Synthetic images from $f_b$ ]{
      \includegraphics[width=0.47\textwidth]{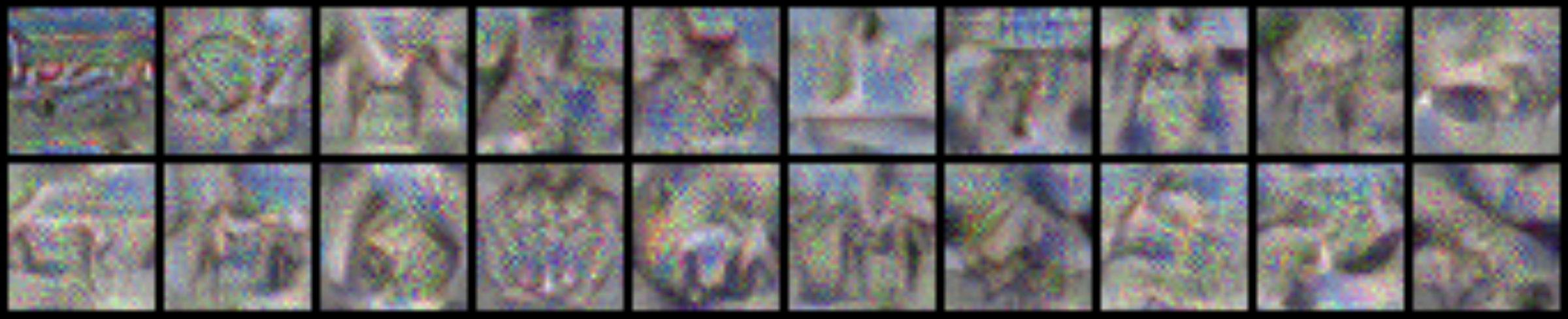}} \\
  \subfigure[\label{fig:fc}Synthetic images from $f_c$]{
      \includegraphics[width=0.47\textwidth]{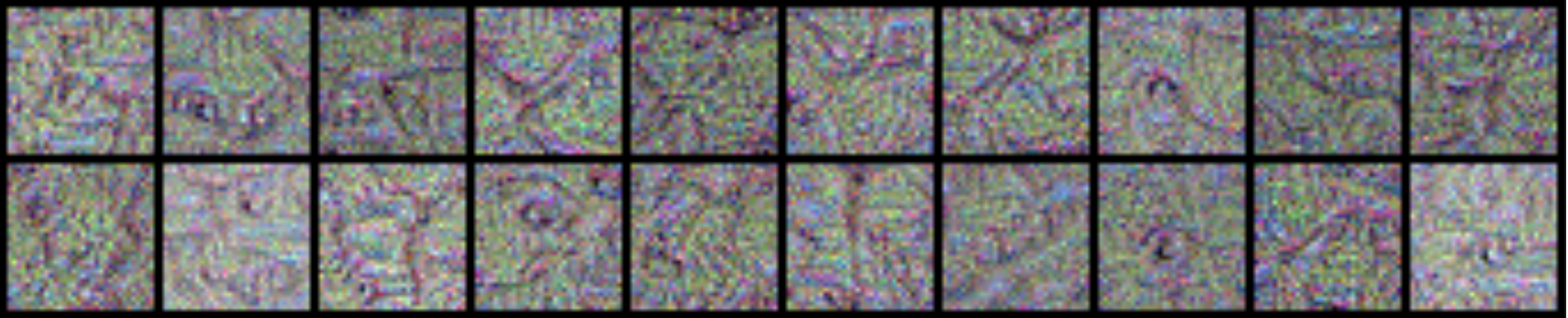}}\\
\caption{\label{fig:dadc}The images recovered from models learned by three different learning algorithms.}
\end{figure}

\subsection{When to start data augmentation}\label{subsec:when_da}
It is important to choose the starting point for when the data augmentation is triggered, as the quality of reconstructed data highly depends on the model performance, which increases as the communication rounds between the server and clients climbs. High-quality augmented data help shrink the divergence of local data distribution among clients and improves privacy, therefore increasing the difficulty for the adversary to tell if certain clients have participated in training. Bad augmented data, such as random noise, can also help maintain privacy, but is likely to erase the useful information in learned model. We study the influence of starting epoch when data augmentation happens. We compare the Fed-ZDAC's fairness performance under the multimodal non-\textit{i.i.d.} setting when the data augmentation starts from global epoch 80, 90, and 95, with the results shown in Figure~\ref{fig:when}. In federated learning, usually longer training epoch leads to solutions with better performance. As a result, the augmented data with higher quality make each client's local data distribution more similar, and contribute to reduce the variance more.

\begin{figure}[t]
	\centering
	\includegraphics[width=0.47\textwidth]{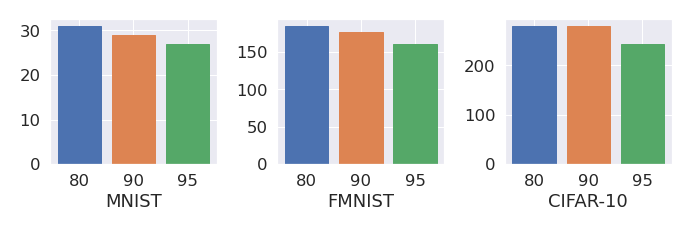}
	\caption{The influence of data augmentation starting point. The horizontal axis is the start global epoch and the vertical axis is the variance.}
	\label{fig:when}
\end{figure}

\section{Conclusions and Future Work}\label{sec:conclusion}
To promote fairness and robustness in federated learning, we propose a federated learning system with zero-shot data augmentation, with possible deployments at the  server (Fed-ZDAS), or at the clients (Fed-ZDAC). We provide a differential privacy analysis. We note that such methods only utilize the statistics of the shared models to generate fake data. Empirical results demonstrate our method achieves both better performance and fairness over commonly used federated learning baselines. For future research, we would like to investigate the combining of Fed-ZDAS and Fed-ZDAC in the same communication round, or at alternate rounds. Similarly, for clarity of the analysis in this paper, we assumed that Fed-ZDAC and Fed-ZDAS are deployed on top of the FedAvg with a simple arithmetic mean aggregation at the server. For future research, we would like to study the effect of deploying ZSDG on top of more complex aggregation schemes.

\subsubsection*{Acknowledgments}
CC is partly supported by the Verizon Media FREP program.

\clearpage
\newpage

{\small
\bibliographystyle{ieee_fullname}
\bibliography{ref}
}

\clearpage
\appendix
\section{Differential Privacy Analysis}\label{app:dp}
We analyze the differential privacy of our proposed methods, adopting the same definition as \cite{wei2020federated} for differential privacy in randomized mechanisms. 
We show that our proposed method satisfies $(0,\delta)$ differential privacy or $(0,\delta)$-DP for short.

\begin{definition}[$(\epsilon, \delta)$-$DP$]\label{def: def1} 
 {A randomized mechanism $\mathcal{M}:\mathcal{X}\rightarrow \mathcal{R}$  with domain $\mathcal{X}$ and range $\mathcal{R}$ satisfies $(\epsilon, \delta)$-$DP$ if for all measurable sets $\mathcal{S}\subset \mathcal{R}$ and for any two adjacent databases $\mathcal{C}$ and $\mathcal{C}^\prime \in \mathcal{X}$,
\begin{align*}
    P(\mathcal{M}(\mathcal{C})\in \mathcal{S})\leq e^\epsilon P(\mathcal{M}(\mathcal{C}^\prime)\in \mathcal{S})+\delta
\end{align*}}
\end{definition}

Since we focus on the client level perspective, the databases $\mathcal{C}$ and $\mathcal{C}^\prime$ here are the sets of clients, which differ on one client only, c and $c^\prime$, {\it i.e.},
\begin{align}
    \mathcal{C}=c \cup \mathcal{C}_0,\nonumber\\
    \mathcal{C}^\prime=c^\prime \cup   \mathcal{C}_0. \label{eq:clients}
\end{align}

Here, we denote the distributions of the datasets $D$ and $D^\prime$ of the two client sets  $\mathcal{C}$ and $\mathcal{C} ^\prime$ as $P_D(X)$ and $P_{D^\prime}(X)$.
Assume both clients start training their models, on their local datasets, starting from 
the same initial parameter $W$, e.g. the global model.  
If their datasets having different distributions, both clients will obtain two different models after local training, which have different parameter distributions. We denote the two parameter distributions as $P_{\mathcal{C}}(W)$ and $P_{\mathcal{C}'}(W)$. For simplicity, we assume the model training is a stochastic process estimating the following posterior distribution according to the Bayes' rule,
\begin{align*}
    P(W|X)\propto P(X|W)P_0(W),
\end{align*}
where $P_0(W)$ is the prior distribution of $W$. Since each client trains on the same model architecture, the likelihood model $P(W|X)$ will be the same for all clients. It is also reasonable to use the same prior distribution for every client. 
\begin{assumption}\label{assump:asump1}
The total variation distance (TV) between the distributions of any two different augmented client datasets are less than $\delta$: $TV(P_D(X),P_{D^\prime}(X))\leq \delta$.\\
\end{assumption}

To verify the assumption~\ref{assump:asump1}, we denote the distribution of generated data as $G$, and the $i$-th client’s dataset is the union of the generated data and the raw data, and the distribution of this combined dataset is denoted as $P_i$. According to the definition of TV distance and its triangle inequality, given an arbitrary $\delta$, we can always generate large enough samples such that $TV(G,P_i)$ is smaller than $\delta/2$. Thus for any two clients,  we have $TV(P_j,P_i)\leq TV(P_j,G)+TV(P_i,G)\leq \delta/2+ \delta/2=\delta$. As a result, the assumption~\ref{assump:asump1} is reasonable.
With the above assumption, we use the data processing inequality stated in Lemma~\ref{lemma:lema1} to derive the TV distance between $P_{\mathcal{C}}(W)$ and $P_{\mathcal{C}^\prime}(W)$.

\begin{lemma}\label{lemma:lema1}
(Theorem 6.2 in \cite{ajjanagadde2017lecture}) {Consider a channel that produces \(Y\) given \(X\) based on the law \(P_{Y \mid X}\) (illustrated in Figure \ref{fig:DPI}). If \(P_{Y}\) is the distribution of \(Y\) when \(X\) is generated by \(P_{X}\) and \(Q_{Y}\) is the distribution of \(Y\) when \(X\) is generated by \(Q_{X},\) then for any \(f\)-divergence \(D_{f}(\cdot \| \cdot),\)
\begin{align*}
    D_{f}\left(P_{Y} \| Q_{Y}\right) \leq D_{f}\left(P_{X} \| Q_{X}\right)
\end{align*}
}
\end{lemma}

\begin{figure}[ht]
	\centering
	\includegraphics[width=0.3\textwidth]{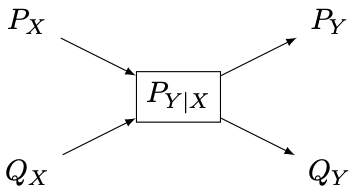}
	\caption{Data processing inequality}
	\label{fig:DPI}
\end{figure}

\begin{theorem}\label{lemma:ZDA_DF}
Federated learning with zero-shot data augmentation satisfies the differential privacy $(0,\delta)$-DP. 
\end{theorem}
\begin{proof}\label{proof:ZDA_DF}
Since the total variation distance is an instance of $f$-divergence \cite{ajjanagadde2017lecture}, applying Lemma~\ref{lemma:lema1}, we obtain
\begin{align*}
    TV(P_{\mathcal{C}}(W),P_{\mathcal{C}^\prime}(W))\leq TV(P_D(X),P_{D^\prime}(X))\leq \delta.
\end{align*}
In federated learning, we perform model aggregation, denoted as $W_{agg}$, as 
\begin{align*}
    W_{agg}=\frac{1}{n}W+\frac{n-1}{n}W_0
\end{align*}
where $W_0$ is the parameter aggregated on the set of other clients  $\mathcal{C}_0$ (as defined in Eq.~\ref{eq:clients}) and $n$ is the number of clients in  $\mathcal{C}$. We denote the two different distributions of $W_{agg}$ in the two models as $P_{\mathcal{C}}(W_{agg})$ and $P_{\mathcal{C}^\prime}(W_{agg})$. Similarly, we can also use the Lemma~\ref{lemma:lema1} to derive that,
\begin{align*}
     TV(P_{\mathcal{C}}(W_{agg}),P_{\mathcal{C}^\prime}(W_{agg}))\leq TV(P_{\mathcal{C}}(W),P_{\mathcal{C}^\prime}(W))\leq \delta
\end{align*}
Based on the definition of total variation distance, we have
\begin{align*}
    \mathop{sup}\limits_{S\subset R}|P_{\mathcal{C}}(W_{agg}\in S)-P_{\mathcal{C}^\prime}(W_{agg}\in S)|\leq \delta
\end{align*}
Define the stochastic mechanism $M$ as the projection from the client set to any model parameter $W_{agg}\in \mathcal{R}$. Then the distribution of $M(\mathcal{C})$ and $M(\mathcal{C^\prime})$ are the distributions of $W_{agg}$ and $W^\prime_{agg}$, respectively. Hence, for any $S\subset R$:
\begin{align*}
    P(M(\mathcal{C})\in S)\leq P(M(\mathcal{C}^\prime) \in S)+ \delta~,
\end{align*}
which finishes the proof that Fed-ZDA satisfies $(0,\delta)$-DP.
\end{proof}

\end{document}

%% file: math_commands.tex
\usepackage{amsmath,amsthm,amsfonts,amssymb}
\usepackage{bm}
\usepackage{graphicx}
\usepackage{sidecap}
\usepackage{mathrsfs}
\usepackage{multirow, multicol}
\usepackage{caption}

\usepackage{wrapfig}
\usepackage{booktabs}

\usepackage{natbib} 
\newtheorem{theorem}{Theorem}[section]
\newtheorem{lemma}[theorem]{Lemma}

\newtheorem{definition}{Definition}[section]
\newtheorem{assumption}{Assumption}[section]

\newcommand{\bbE}{\mathbb{E}}

\newcommand{\calD}{\mathcal{D}}

\def\1{\bm{1}}

\def\vx{{\bm{x}}}

\DeclareMathAlphabet{\mathsfit}{\encodingdefault}{\sfdefault}{m}{sl}
\SetMathAlphabet{\mathsfit}{bold}{\encodingdefault}{\sfdefault}{bx}{n}

